\documentclass{article} 
\usepackage{style2027/iclr2027_conference,times}

\usepackage{amsmath,amsfonts,bm}

\def\secref#1{section~\ref{#1}}

\def\eqref#1{equation~\ref{#1}}

\def\Algref#1{Algorithm~\ref{#1}}

\def\1{\bm{1}}

\DeclareMathAlphabet{\mathsfit}{\encodingdefault}{\sfdefault}{m}{sl}
\SetMathAlphabet{\mathsfit}{bold}{\encodingdefault}{\sfdefault}{bx}{n}

\usepackage{graphicx}
\usepackage{wrapfig}
\usepackage{placeins}
\usepackage{enumitem}
\usepackage{xcolor}
\usepackage{booktabs}
\usepackage{float}
\usepackage{amsmath}
\usepackage{amsthm}
\newtheorem{theorem}{Theorem}
\newtheorem{lemma}[theorem]{Lemma}
\newtheorem{definition}[theorem]{Definition}
\usepackage{algorithm}
\usepackage{algpseudocode}


\usepackage{hyperref}
\usepackage{url}

\title{zkLLMPoT: Efficient Zero Knowledge Proof of Training for Large Language Models}

\author{%
Junkai Liang$^{1,3}$, Zhanpeng Guo$^{2}$, Pengfei Wu$^{1,3}$, Qingni Shen$^{1}$, \\
\bf Jiaheng Zhang$^{4}$, Zhonghai Wu$^{1}$, Haiyang Xue$^{3}$, Shengfang Zhai$^{4}$ \\[0.4em]
{\normalfont $^{1}$\,Peking University} \\
{\normalfont $^{2}$\,China Telecom Quantum Information Technology Group Co., Ltd} \\
{\normalfont $^{3}$\,Singapore Management University} \\
{\normalfont $^{4}$\,National University of Singapore}%
}

\iclrfinalcopy 
\begin{document}

\maketitle  

\begin{abstract}
Auditing the claimed outcomes of large language model (LLM) training is challenging when model weights and training data are private, while cryptographically proving the full training process is prohibitively expensive at Transformer scale. We present zkLLMPoT, a zero-knowledge framework that certifies auditor-defined properties of a trained checkpoint through forward evaluation rather than verification of its optimization trajectory. zkLLMPoT includes 2 phases: 1) The trainer fixes the architecture and the model weights are committed. Then the auditor selects challenge sequences, preventing the trainer from modifying the checkpoint in response to the audit data. 2) Then the trainer proves the objective value attained by the committed model on those sequences. This formulation makes the certification cost independent of the number of training iterations, without revealing model weights or requiring access to private training data. We build on sumcheck- and lookup-based arguments to certify Transformer computations, while supporting next-token loss and task-specific audit objectives. Across four model families, operator-level benchmarks yield proving times of 41–59 seconds for 1.1–1.5B-parameter models and 131 seconds at 13B for the covered operators, with verification below half a second at a sequence length of 512.

\end{abstract}

\section{Introduction}

Large language models (LLMs) increasingly serve as critical infrastructure for language understanding, generation, and decision support~\citep{brown2020language}, yet assessing their training outcomes raises pressing \emph{privacy} and \emph{trustworthiness} concerns. Trained models can memorize and leak sensitive examples~\citep{shokri2017membership,carlini2021extracting}, so institutions may withhold training data and model weights from external auditors, leaving the public with only a model card or an API endpoint.
Consequently, when a trainer announces a new LLM, users and regulators often cannot independently determine whether the released checkpoint actually satisfies the capability or policy claims attributed to its training~\citep{waiwitlikhit2024zkaudit,srivastava2024optimistic}.
Without a cryptographic certificate tied to the checkpoint, such claims remain difficult to audit despite their legal, commercial, and safety consequences.

Existing approaches to trustworthy LLM computation leave this gap largely unresolved.
First, a substantial line of security work develops zero-knowledge proofs for \emph{inference} integrity, ranging from CNN and general neural-network predictors~\citep{liu2021zkcnn,weng2021mystique,hao2024scalable,chen2024zkml} to LLM-oriented systems such as zkLLM and zkGPT~\citep{sun2024zkllm,qu2025zkgpt}.
These systems target per-query serving integrity: they certify that a committed model produced a claimed output, rather than providing a checkpoint-level audit against an objective and challenge set chosen by an auditor after the model is committed.
Second, cryptographic proofs of training that do target the learning process typically encode full gradient descent with backpropagation.
Even for non-transformer networks this is already expensive.
For example, Kaizen requires about 15 minutes of proving time per training iteration on a VGG-11 model~\citep{abbaszadeh2024kaizen}, and related certificates such as confidential proofs of differentially private training can take on the order of 100 hours even on CIFAR-10~\citep{shamsabadi2024confidential}.
At LLM scale, naively proving every backward pass is therefore impractical with current techniques.
Third, early practical zkPoT protocols are largely restricted to demonstration-level models such as logistic regression~\citep{garg2023zkpot}, while optimistic replication-based verifiable training still assumes a third-party auditor who can re-execute training~\citep{srivastava2024optimistic,waiwitlikhit2024zkaudit}.
Taken together, prior work either targets per-query inference, proves the training trajectory through an expensive backpropagation circuit, or scales only to demonstration-level learners. Efficient checkpoint-level, commit-then-challenge auditing for LLM training outcomes therefore remains largely open.

In this work, we present \textbf{zkLLMPoT}, a zero-knowledge framework for certifying properties of a checkpoint that are claimed as outcomes of LLM training or fine-tuning.
Figure~\ref{fig:overview} shows the audit workflow.
In the full protocol, zkLLMPoT lets a trainer prove that a committed checkpoint attains a declared objective on auditor-chosen data without exposing the weights or making the private training corpus part of the proved relation. The same interface can express a declared group-fairness metric or a per-example loss used as a memorization and membership-risk signal~\citep{waiwitlikhit2024zkaudit}.
Built on zk-SNARKs, the proof size and verification time are independent of the number of training iterations and do not require re-executing training. Our prototype measures the principal Transformer operators needed by this relation and provides full circuit composition and zero-knowledge blinding.

\begin{figure}[t]
\centering
\includegraphics[width=\linewidth]{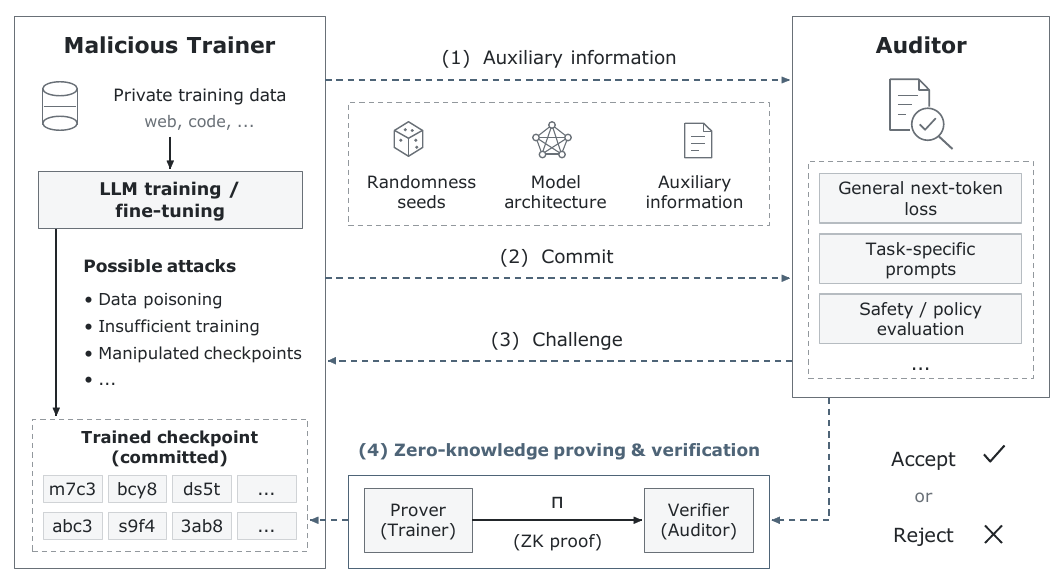}
\caption{Audit workflow of zkLLMPoT. After publishing public auxiliary information, the trainer commits a checkpoint. The auditor then issues challenge sequences under a chosen objective (next-token loss, task prompts, or a safety check). The auditor verifies a zero-knowledge proof and accepts or rejects the claimed outcome without learning the private training data or model weights.}
\label{fig:overview}
\end{figure}

Our protocol attests to properties of a committed checkpoint of a large language model. It can show that (i) the checkpoint improves over a committed reference state on auditor-chosen sequences; (ii) it attains a declared capability on domain-specific tasks; and (iii) the evaluation follows the declared Transformer architecture and objective. The zero-knowledge protocol hides the weights and intermediate activations. Achieving such a certificate raises two challenges. First, proving the training process itself is prohibitively slow: every step adds a forward pass, a backward pass, and an optimizer update to the proved relation, so the proving time grows with the length of the training run. Second, if the evaluation data are known before the checkpoint is fixed, the trainer can overfit the checkpoint to those data, and a good result on them no longer reflects the claimed training outcome.
To make the proof practical, zkLLMPoT evaluates the committed checkpoint through a loss-function-style objective. This is a feasible measure of training outcomes, because the loss a checkpoint attains on held-out data from the target distribution is the standard way to judge whether pretraining or fine-tuning achieved its goal~\citep{gururangan2020dontstop}. To avoid the trainer overfitting the evaluation, we specified a data commitment protocol.
The auditor specifies the objective, the trainer freezes the corresponding circuit and commits the checkpoint, and the auditor then chooses the evaluation sequences. The trainer is not able to modify her weights after seeing the commitment. Together, these techniques let zkLLMPoT generate and verify proofs quickly, with proving times of $41$--$59$\,s for the covered operators on $1.1$--$1.5$B-parameter models and $131$\,s at $13$B, and verification below half a second at sequence length $512$, while the commit-then-challenge order removes the trainer's opportunity to overfit the checkpoint to the auditor's sequences.

In summary, our contributions are:
\begin{itemize}[leftmargin=0.9em]
\item An outcome-attestation formulation that audits properties attributed to LLM training through a loss-function-style evaluation of committed weights, excluding the backward pass and optimizer from the proved relation;
\item A commit-then-challenge protocol in which the trainer binds the circuit and checkpoint before the auditor selects the evaluation sequences, which then enter the statement as public inputs;
\item Operator-level optimizations and fusion for compiling Transformer layers into compact ZK circuits, plus an interface for auditor-specified terminal objectives;
\item A component-based study across four model families, yielding proving times of $41$--$59$\,s for the covered operators at the $1.x$B scale and $131$\,s at $13$B, with verification below half a second.
\end{itemize}

\section{Preliminaries}

In this section, we introduce the properties a zero-knowledge protocol is required to satisfy, and the threat model of our method.
We first recall completeness, soundness, zero knowledge, and succinctness for the evaluation relation in~\secref{sec:prelim-zksnark}.
We then state the two-party threat model in~\secref{sec:prelim-threat}.

\subsection{Zero Knowledge Proof}
\label{sec:prelim-zksnark}

A \emph{zero-knowledge succinct non-interactive argument of knowledge} (zk-SNARK) lets a prover convince an auditor that it knows a witness $w$ for a public statement $x$ of an NP relation $R$, without revealing $w$ and without the auditor re-executing the computation~\citep{goldwasser1989knowledge,bitansky2012succinct,groth2016snark}.
Here $R$ is a set of pairs $(x,w)$, and its associated language $L_R = \{\,x \mid \exists\, w : (x,w)\in R\,\}$ is the set of statements that admit a witness.
We instantiate $R$ with the definition of model training: the public input is the declared architecture $\mathcal{A}$, the objective, the evaluation sequences $S$, a commitment to the released weights $\theta$, and a claimed scalar $v$; the witness is $\theta$ and the intermediate activations, with a reference checkpoint added when the statement compares two model states.
An accepting proof convinces the auditor that $v$ is the value the stated objective, such as the output of a loss function, takes on $f_{\mathcal{A},\theta}(S)$.
We require the following properties for the evaluation relation above.

\begin{itemize}
\item \textbf{Completeness.}
If the committed weights, public sequences, activations, and objective value satisfy the declared forward relation, then an honest prover can produce a proof that an honest auditor accepts.

\item \textbf{Soundness.}
For every (efficient) malicious prover $\mathcal{P}^*$ and every statement $x \notin L_R$, the verifier accepts $\mathcal{P}^*$'s proof with only negligible probability, which means a trainer cannot substitute uncommitted weights, change the declared architecture or objective, or claim a value $v$ that the committed $\theta$ does not attain on the public $S$.

\item \textbf{Zero knowledge.}
The proof reveals nothing about the witness beyond the truth of the statement~\citep{goldwasser1989knowledge}.
This means the auditor learns that the committed $\theta$ satisfies the evaluation relation but learns no additional information about the weights or intermediate activations.

\item \textbf{Succinctness.}
Proof size and verification time grow more slowly than the computation represented by the relation~\citep{bitansky2012succinct,groth2016snark}.
Succinctness makes an outcome certificate transferable because downstream users can verify it without matching the prover's GPU budget.
\end{itemize}

\subsection{Threat Model}
\label{sec:prelim-threat}

We use the two-party layout of zkPoT~\citep{garg2023zkpot}.
A trainer holds a checkpoint $\theta$ and any private training corpus as intellectual property.
An auditor knows the public architecture $\mathcal{A}$ and the objective $\mathcal{L}$, and later chooses the evaluation sequences $S$.
Both parties are computationally bounded.

\begin{itemize}
\item \textbf{Malicious trainer.}
The trainer need not follow training or the protocol.
Its goal is to make the auditor accept a false statement about the committed checkpoint, for example by substituting uncommitted or weaker weights, changing $\mathcal{A}$ or $\mathcal{L}$, adapting $\theta$ after seeing $S$, or claiming a value $v$ that the committed $\theta$ does not attain on the public $S$.
This matches the zkPoT requirement that a false training claim be rejected~\citep{garg2023zkpot}.
Soundness of~\secref{sec:prelim-zksnark} is required against this trainer.

\item \textbf{Semi-honest auditor.}
As in zkPoT~\citep{garg2023zkpot}, the auditor follows the verification procedure and reports accept or reject correctly, but may try to recover $\theta$, intermediate activations, or the private training corpus from the transcript.
Zero knowledge is required against this auditor, so the proof reveals only the truth of $v=\mathcal{L}(f_{\mathcal{A},\theta}(S))$.
\end{itemize}

Unlike Garg et al.'s zkPoT statement, which asserts that a committed model was obtained by running a declared learning algorithm on a committed dataset~\citep{garg2023zkpot}, our relation encodes only the committed $\theta$ satisfy the loss function defined by the auditor.
The protocol therefore binds the published checkpoint to the auditor's sequences and loss function choices and does not bind full backpropagation.

\section{zkLLMPoT}
\label{sec:llmpot}

In this section, we detail zkLLMPoT, our protocol for certifying an auditor-chosen evaluation of a committed LLM checkpoint. We first formulate a unified decoder-only training model and the outcome claim in~\secref{sec:prelim-training}. We then describe the protocol components, including polynomial commitments, linear and lookup gadgets, the commit-then-challenge order, and auditor-specified objectives, in~\secref{sec:llmpot-protocol}. We finally combine these components into the end-to-end protocol and state its soundness in~\secref{sec:llmpot-together}.

\subsection{Formulation}
\label{sec:prelim-training}
Contemporary language models such as OPT~\citep{zhang2022opt}, Llama~\citep{touvron2023llama}, Qwen2.5~\citep{bai2023qwen}, and DeepSeek-Coder~\citep{deepseek2024v2} are decoder-only Transformers.
They differ in normalization, positional encoding, attention layout, and feed-forward activation, but they share one computational skeleton and one training objective.
We therefore treat them as instances of a single architecture family $\mathcal{A}$, rather than writing a separate training model for each.

\textbf{Forward computation.}
Fix a vocabulary $\mathcal{V}$ and a token sequence $x=(x_1,\ldots,x_n)\in\mathcal{V}^n$.
A model in $\mathcal{A}$ is specified by weights $\theta$ together with a small operator choice $(\mathrm{Norm},\mathrm{Pos},\mathrm{Attn},\Phi)$, and computes
\begin{equation}
\label{eq:llm-forward}
  h^{(0)}=\mathrm{Embed}_\theta(x),\qquad
  h^{(\ell+1)}=\mathrm{Block}_\theta^{(\ell)}(h^{(\ell)}),\qquad
  z=h^{(L)}W_{\mathrm{out}}.
\end{equation}
Every block uses the pre-norm residual form common to all of the above families,
\begin{equation}
\label{eq:llm-block}
  u = h+\mathrm{Attn}_\theta\bigl(\mathrm{Norm}(h);\mathrm{Pos}\bigr),\qquad
  h' = u+\mathrm{FFN}_\theta^{\Phi}\bigl(\mathrm{Norm}(u)\bigr).
\end{equation}
Here $\mathrm{Norm}$ is LayerNorm (OPT) or RMSNorm (Llama, Qwen2.5, DeepSeek-Coder);
$\mathrm{Pos}$ is either an additive absolute embedding or rotary embedding (RoPE) inside attention;
$\mathrm{Attn}$ is multi-head, grouped-query, or multi-head latent attention;
and $\Phi$ specializes the feed-forward network (FFN) to ReLU (OPT) or SwiGLU (Llama, Qwen2.5, DeepSeek-Coder).
Row-wise softmax turns the logits $z$ into next-token distributions $p_\theta(\cdot\mid x_{\le i})=\mathrm{softmax}(z_i)$.

\textbf{Training objective and updates.}
Pre-training and supervised fine-tuning of all these models minimize a masked next-token negative log-likelihood.
Given a sequence $x$ and a binary mask $m\in\{0,1\}^n$ that marks which tokens are predicted,
\begin{equation}
\label{eq:llm-nll}
  \mathcal{L}_{\mathrm{LM}}(\theta;x,m)
  \;=\;
  -\sum_{i=1}^{n-1} m_{i+1}\,\log p_\theta(x_{i+1}\mid x_{\le i}).
\end{equation}
Causal pre-training takes $m\equiv 1$ (except the first token); instruction or SFT training zeros $m$ on the prompt and keeps it on the response.
The same loss therefore covers OPT-style continued pre-training and Llama~/~Qwen2.5~/~DeepSeek-Coder-style supervised fine-tuning.
A training run is the first-order iteration
$\theta_{t+1}=\mathsf{Opt}(\theta_t,\nabla_\theta\mathcal{L}_{\mathrm{LM}}(\theta_t;x_t,m_t),\eta_t,s_t)$,
starting from an initialization $\theta_0$ and releasing a checkpoint $\theta_T$.
Here $\mathsf{Opt}$ is a standard adaptive optimizer (e.g., AdamW) with learning rate $\eta_t$ and state $s_t$.
Preference alignment (PPO, DPO, GRPO) only replaces $\mathcal{L}_{\mathrm{LM}}$ by another differentiable scalar on one or two forwards of the same map~\eqref{eq:llm-forward}; the Transformer itself is unchanged.

Our protocol zkLLMPoT confidentially and efficiently certifies an evaluation outcome attributed to training or fine-tuning an open decoder-only LLM, without replaying $\mathsf{Opt}$ or inspecting $\theta$.
zkLLMPoT compiles this model into a zk-SNARK certificate).
Writing $f_{\mathcal{A},\theta}$ for the forward map and $\mathcal{L}$ for a scalar objective, the claim is
\begin{equation}
\label{eq:llmpot-claim}
  v \;=\; \mathcal{L}\bigl(f_{\mathcal{A},\theta}(S)\bigr).
\end{equation}
The trainer and auditor agree on $\mathcal{A}$, and the protocol proves that the committed weights $\theta$ and auditor-chosen sequences $S$ produce the claimed value $v$.
zkLLMPoT has four components: polynomial commitments that bind the witness tables, constraint gadgets for the Transformer forward pass, a commit-then-challenge order for auditor-chosen data, and an auditor-chosen scalar objective.
We describe each component first and then combine them into the end-to-end protocol.

\begin{algorithm}[t]
\caption{zkLLMPoT main protocol}
\label{alg:main}
\small
\begin{algorithmic}[1]
\Require Architecture $\mathcal{A}$; commitment scheme $\{\mathsf{Commit},\mathsf{Open}\}$ over $\mathbb{F}_p$.
\Ensure The auditor accepts or rejects $v=\mathcal{L}\bigl(f_{\mathcal{A},\theta}(S)\bigr)$ for the released $\theta$.
\State The auditor fixes $\mathcal{L}$ and an acceptance band $[\underline{v},\overline{v}]$. The trainer compiles $(\mathcal{A},\mathcal{L})$ into $C_{\mathcal{A},\mathcal{L}}$.
\State The trainer sends $[[\theta]]\gets\mathsf{Commit}(\theta)$, binding the weights before the auditor's challenge set $S$ is revealed.
\State The auditor chooses $S$ and sends it to the trainer.
\State The trainer evaluates $v\gets C_{\mathcal{A},\mathcal{L}}(\theta,S)$ and sends $[[v]]$, together with commitments $[[H]]$ to the intermediate activations $h^{(0)},\dots,h^{(L)}$ of~\eqref{eq:llm-forward} that the layer arguments open. $S$ remains a public input.
\State For $\ell=L,\dots,1$ the trainer reduces the output claim $h^{(\ell)}(u_\ell)=y_\ell$ through the gadget of $\mathrm{Block}^{(\ell)}$ in~\eqref{eq:llm-block} to an input claim $h^{(\ell-1)}(u_{\ell-1})=y_{\ell-1}$ and a weight claim $\theta_\ell(w_\ell)=c_\ell$, sending each sumcheck transcript.
\State The auditor replays every reduction and checks the chain: the input claim $\bigl(u_{\ell-1},y_{\ell-1}\bigr)$ emitted at layer $\ell$ is the output claim it verifies at layer $\ell-1$.
\State The auditor checks $\mathsf{Open}$ for $[[\theta]]$ and $[[H]]$ at the challenge points, evaluates leftover claims on the public table $S$ itself, and recomputes each public table $T\colon\{0,1\}^{n}\to\mathbb{F}_p$ of the nonlinear gadgets instead of accepting an opening of $T$.
\State The auditor accepts iff every check passes and $v\in[\underline{v},\overline{v}]$.
\end{algorithmic}
\end{algorithm}
\subsection{Protocol Components}
\label{sec:llmpot-protocol}

\textbf{Commitment design.}
We compile~\eqref{eq:llm-forward}--\eqref{eq:llm-nll} into polynomial constraints over $\mathbb{F}_p$, giving the public circuit $C_{\mathcal{A},\mathcal{L}}$ whose description or digest is fixed before the challenge rather than committed with the witness.
The witness fills that circuit: $W=(\theta,H,v)$, the weights, the layer activations $H=(h^{(0)},\ldots,h^{(L)})$, and the claimed scalar.
The sequences $S$ are a public input: both parties know them after the auditor sends them, and leftover claims about $S$ are evaluated on that public table.
The trainer sends $[[\theta]]$, $[[H]]$, and $[[v]]$ under a polynomial commitment scheme~\citep{thaler2022proofs}.
Each tensor is a length-$2^n$ table over $\mathbb{F}_p$ identified with its unique multilinear extension, so $[[T]]$ commits to that polynomial rather than to an unstructured string, and proving work is linear in the table size.
Weight roots are computed once per checkpoint and reused; activation roots are computed per audit.
The polynomial commitment scheme is detailed in~\secref{app:zk-pcs}, and the concrete definition of each operator and its constraint gadget is given in~\secref{app:zk-ops}.

\textbf{Constraint design over $\mathbb{F}_p$.}
Tensors in~\eqref{eq:llm-forward} through~\eqref{eq:llm-nll} become tables $T\colon\{0,1\}^n\to\mathbb{F}_p$ by signed fixed-point quantization at a public scale $s$, reduction modulo $p$, and padding to length $2^n$.
We then compile $f_{\mathcal{A},\theta}$ by the standard split of a neural network into \emph{linear} and \emph{nonlinear} layers~\citep{liu2021zkcnn,weng2021mystique}, giving each class a constraint gadget.

The embedding selection, the projections $W_q,W_k,W_v$, the attention products $QK^{\top}$ and $PV$, the FFN maps, $W_{\mathrm{out}}$, and the residual additions in~\eqref{eq:llm-block} admit low-degree constraints over $\mathbb{F}_p$.
We prove the matrix products with hypercube inner-product sumchecks and use direct algebraic constraints for additions and pointwise products, all against committed weight and activation tables.
The functions $\mathrm{Norm}$, $\Phi$, and row-wise softmax also require nonlinear operations such as inverse square roots and exponentials.
We combine low-degree constraints with lookups into public tables for these functions; the tables belong to $C_{\mathcal{A},\mathcal{L}}$ and are not witness commitments.
Copy constraints identify each producer table with the next consumer along~\eqref{eq:llm-block}, so the gadgets compose into one circuit.

We compose the embedding, $L$ residual blocks, output head, and $\mathcal{L}$ gadgets to obtain $C_{\mathcal{A},\mathcal{L}}$.
A different $(\mathrm{Norm},\mathrm{Pos},\mathrm{Attn},\Phi)$ or a different scalar $\mathcal{L}$ only swaps the gadget at that place in the chain.
We prove on an integer-quantized copy of the weights, which avoids encoding floating-point arithmetic as constraints and thereby reduces that overhead.

\textbf{Data-related protocol and commitments.}
The order of the interaction prevents the trainer from adapting the checkpoint to the evaluation set.
The circuit \(C_{\mathcal{A},\mathcal{L}}\) and the binding commitment \([[\theta]]\) are fixed before the auditor reveals \(S\), so the trainer cannot replace them with a circuit or checkpoint tailored to those sequences.
The sequences \(S\) enter the statement as a public input, so leftover embedding claims are checked against the auditor's \(S\) rather than against a committed table.
Consequently, soundness forces the accepted value to be \(v=\mathcal{L}(f_{\mathcal{A},\theta}(S))\) for the already committed \(\theta\) on that public \(S\), rather than for an adaptively chosen \(\theta'(S)\) or a substituted \(S'\).

\textbf{Custom loss functions.}
The auditor chooses the scalar objective \(\mathcal{L}\) in~\eqref{eq:llmpot-claim}, while the committed \(\theta\) and the forward map \(f_{\mathcal{A},\theta}\) remain unchanged.
Only the final objective gadget of \(C_{\mathcal{A},\mathcal{L}}\) changes, and both that gadget and \(\theta\) are fixed before the auditor selects \(S\).
The default is next-token NLL~\eqref{eq:llm-nll} (or the same sum with a prompt mask for SFT); the opened \(v\) can be turned into perplexity \(\mathrm{PPL}=\exp(v/|m|)\) in the clear.
On a labeled set the auditor may instead use \(\mathcal{L}_{\mathrm{CE}}=-\sum\log p_{\theta}(y\mid x)\).
A fairness check is the gap \(v=\mathcal{L}(\theta;S_0)-\mathcal{L}(\theta;S_1)\) between two slices~\citep{waiwitlikhit2024zkaudit}.
A membership check is the NLL of one challenge sequence \(x^{\star}\), with a small \(v\) being the usual memorization signal~\citep{shokri2017membership,carlini2021extracting}.
A safety check uses the same CE on a refusal or toxicity slice, or a public score \(v=\sum_i\sigma(z_i)\) on the logits.
In every case the trainer still proves \(v=\mathcal{L}(f_{\mathcal{A},\theta}(S))\); only the public \(\mathcal{L}\) changes.

\subsection{Putting It All Together}
\label{sec:llmpot-together}

\Algref{alg:main} combines the four components into one end-to-end protocol.
The auditor fixes the objective, the trainer binds the circuit and released weights, and only then does the auditor choose the evaluation sequences.
The trainer evaluates the committed checkpoint and reduces the resulting operator claims through the corresponding gadgets; the auditor checks that those claims chain into one forward evaluation and that the final value lies in the acceptance band.
The protocol above focuses on the two-party interaction; the lookup arguments and the linear-layer operator design details are given in~\secref{app:zk-ops}.

\begin{theorem}[Soundness of \Algref{alg:main}]
Let $\theta$ be the checkpoint committed in Step~2 and $S$ the sequences the auditor draws afterwards. A prover that makes the auditor accept a value $v\neq\mathcal{L}(f_{\mathcal{A},\theta}(S))$ succeeds with probability at most
\begin{equation}
  \varepsilon \;\le\; \sum_{a\in\mathcal{C}_{\mathcal{A},\mathcal{L}}}\varepsilon_a
  \;+\; \varepsilon_{\mathrm{ipa}}^{\mathrm{tot}}
  \;+\; \mathrm{Adv}^{\mathrm{DL}}_{\mathbb{G}}(t) ,
\end{equation}
where $\varepsilon_a$ is the sumcheck error of a linear gadget or the lookup error of a nonlinear gadget, and $\varepsilon_{\mathrm{ipa}}^{\mathrm{tot}}$ is the aggregate soundness error of all inner-product opening arguments.
\end{theorem}
The theorem bounds the probability that the auditor accepts a value other than the loss of the committed checkpoint on the sequences chosen afterwards.
In practice, an accepting transcript means that \(v=\mathcal{L}(f_{\mathcal{A},\theta}(S))\) for the \(\theta\) fixed in Step~2, except with negligible probability.
The bound comes from three sources: the sumcheck or lookup error of each gadget, the aggregate error of the inner-product openings, and the discrete-log binding of the commitment.
A numerical evaluation of the bound shows approximately \(127\) bits of security for the four \(1.x\)B checkpoints at \(T{=}512\), under a conservative union bound in which the statistical error stays below \(2^{-209}\) and each computational term is about \(2^{-128}\).
The detailed proof and this calculation are given in~\secref{app:zk-soundness}.

\section{Experiments}

In the experimental stage we consider two dimensions.
\textbf{Effectiveness}: If verification completes as specified given an honest two-party end-to-end protocol.
\textbf{Efficiency} measures the cost incurred by both parties.
We benchmark the implemented operators in CUDA with a CPU-side verifier and aggregate those measurements according to each model's exact scale and operator counts.

\subsection{Experimental Setup}
\label{sec:eval-setup}

\textbf{Workloads and model scales.}
The efficiency benchmarks use batch size one, sequence length $T{=}512$, and synthetic fixed-point tensors at the dimensions of the public model configurations.
The scaling study covers fourteen configurations from four families: OPT~\citep{zhang2022opt}, Llama~\citep{touvron2023llama}, Qwen2.5~\citep{bai2023qwen}, and DeepSeek-Coder~\citep{deepseek2024v2}. Their nominal sizes span $0.125$B to $13$B parameters.

\textbf{Baselines.}
We compare against prior zero-knowledge proofs of training (zkPoT) and confidential DP-SGD certification.
Garg et al.'s logistic-regression zkPoT~\citep{garg2023zkpot} is the only one of these with a public implementation, and we re-execute it on our own hardware.
Public implementations of Kaizen, ZKAudit, and Confidential-DPproof~\citep{shamsabadi2024confidential} are unavailable, so we use the costs reported in the corresponding papers and mark them as estimates.

\textbf{Hardware and software.}
All operator measurements used a dedicated NVIDIA A100-PCIE (40\,GB) server with an Intel Xeon CPU (80 threads), 629\,GB of RAM, Ubuntu~22.04, and CUDA~12.4. The mcl verifier ran on the same server's CPU without using the GPU.
The GPU was not shared during the reported measurements. Contention can dominate protocol time, and we observed slowdowns of up to $70{\times}$ on the same configuration when other tenants were active.

\textbf{Metrics.}
For each configuration we report the proving time, verification time, and proof size of the covered operators. The proving times are aggregated from per-operator benchmarks measured on the A100, as described above.
We report commitment time separately from proving time. The commitment to released weights is paid once when the model is published and can be reused across later audits, while commitments to activations are generated for each audit.

\subsection{Implementation}

The prototype contains a CUDA/C++ prover and a separate C++ verifier.
The prover reuses the BLS12-381 field and curve kernels of zkLLM~\citep{sun2024zkllm} for GPU arithmetic. On top of them, we implement Pedersen commitments in a Hyrax-style layout, sumcheck arguments for linear operators, and a logup-style lookup argument for nonlinear functions.
The verifier is a separate C++ binary that links mcl for BLS12-381 arithmetic and contains no CUDA.
The protocol supports cross-entropy and custom terminal objectives, and the fairness, membership, and safety audits reuse the same circuit through the objective interface.
Model dimensions and operator counts come from the public configurations, while benchmark inputs are synthetic fixed-point integers.

\subsection{Main Results}

Table~\ref{tab:main} compares zkLLMPoT with prior training-oriented ZK systems under each paper's own model and dataset.
Each row names the system and the model it was evaluated on, then reports the parameter count and four costs.
\textit{Commit} is the time to bind the released weights; when a cell says ``in prove'', that cost is folded into proving.
\textit{Prove} is the time to produce the certificate, \textit{Verify} is the auditor's checking time, and \textit{Proof} is the size of the transferable object.
Kaizen and ZKAudit quote these costs per training iteration, so a fine-tuning run multiplies them by the step count: Kaizen needs $882$\,s per iteration on VGG-11, and DPproof reports $100$ hours for DP-SGD on CIFAR-10 features.
DPproof is interactive, so its auditor stays online throughout and produces no proof object.
The zkLLMPoT rows are one-shot results at $T{=}512$ on $1.1$--$1.5$B checkpoints: commitment is $138$--$206$\,s, proving $41$--$59$\,s, verification below half a second, and proofs about $17$--$20$\,MiB.
The relation contains a single forward evaluation, so the cost does not grow with the number of training steps, and the evaluated checkpoints have two orders of magnitude more parameters than the largest prior system in the table.
zkLLMPoT is therefore more efficient as a training-outcome certificate: it certifies a much larger model in less time than prior systems spend on one iteration of a much smaller learner.

\begin{table}[t]
\caption{Comparison with prior proof-of-training systems~\citep{garg2023zkpot,abbaszadeh2024kaizen,waiwitlikhit2024zkaudit,shamsabadi2024confidential}, using the costs reported in the corresponding papers. Kaizen and ZKAudit charge per training iteration; DPproof is interactive, so its auditor is online throughout and produces no proof object. ``in prove'' folds commitment into proving. Our rows are at sequence length $T{=}512$.}
\label{tab:main}
\centering
\small
\resizebox{\linewidth}{!}{%
\begin{tabular}{llrrrrr}
\toprule
System & Model & Params & Commit & Prove & Verify & Proof \\
\midrule
zkPoT & logistic regression & $1$K & $3690$\,s & $600$\,s & $30$\,s & $350$\,MB \\
Kaizen & VGG-11 & $10.1$M & $218$\,s & $882$\,s & $130$\,ms & $1.63$\,MB \\
ZKAudit & MobileNet v2 & $3.5$M & in prove & $328$\,s & $6.0$\,ms & $9.03$\,kB \\
DPproof & logistic regression & $41$K & in prove & $100$\,h & $100$\,h & N/A \\
\midrule
zkLLMPoT (ours) & OPT-1.3B & $1.3$B & $175$\,s & $58.8$\,s & $428$\,ms & $17.0$\,MiB \\
zkLLMPoT (ours) & Llama-1.1B & $1.1$B & $138$\,s & $53.2$\,s & $423$\,ms & $17.3$\,MiB \\
zkLLMPoT (ours) & Qwen2.5-1.5B & $1.5$B & $206$\,s & $42.4$\,s & $416$\,ms & $19.9$\,MiB \\
zkLLMPoT (ours) & DeepSeek-Coder-1.3B & $1.3$B & $171$\,s & $41.0$\,s & $406$\,ms & $16.9$\,MiB \\
\bottomrule
\end{tabular}}
\end{table}

\subsection{Operator Efficiency}

\begin{wrapfigure}{r}{0.55\columnwidth}
\vspace{-6pt}
\centering
\includegraphics[width=\linewidth]{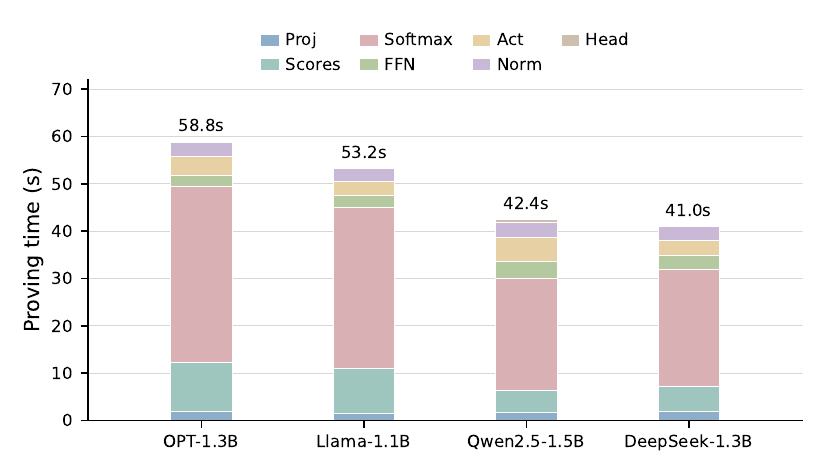}
\caption{Per-operator proving time at $T{=}512$ for the four $1.x$B checkpoints of Table~\ref{tab:main}.}
\label{fig:operators}
\end{wrapfigure}
Figure~\ref{fig:operators} decomposes the Table~\ref{tab:main} proving times into the covered gadgets at $T{=}512$.
Fused sumchecks for the linear projections, FFN maps, and output head remain inexpensive, $4.1$--$6.0$\,s in total on every $1.x$B checkpoint.
Normalization costs $2.6$--$3.1$\,s and the elementwise activation $3.0$--$5.0$\,s, both lookups over tensors of width $T$.
Attention works on the $\mathrm{heads}\times T\times T$ grid, so it is heavier: score matmuls range from $4.6$\,s on Qwen2.5-1.5B ($12$ heads) to $10.3$\,s on OPT-1.3B ($32$ heads), and softmax is the largest term, $56$--$64\%$ of proving time.
Composing these gadgets into one forward certificate is what yields the $41$--$59$\,s totals of Table~\ref{tab:main}: each operator class is efficient at this scale, and assembling them already reaches practical one-shot proving performance.

\subsection{Generalization to Broader Model Scales}

Figure~\ref{fig:scaling} shows the efficiency of zkLLMPoT across four model families spanning $0.125$B to $13$B parameters.
We aggregate the dense-weight commitment and covered proving workload for a $512$-token challenge as a function of model size.
Commitment is approximately linear in the number of committed dense weights. Across these configurations it corresponds to between $124$ and $142$ seconds per nominal billion parameters, with modest variation from the tensor mix in each architecture.
The covered proving workload grows more slowly than model size: increasing the nominal parameter count by $104{\times}$ increases the proving time by $8.3{\times}$, from $15.7$ to $130.6$ seconds.
Verification stays under half a second throughout.
These results indicate favorable scaling for the covered operators.

\begin{figure}[t]
\centering
\includegraphics[width=0.49\linewidth]{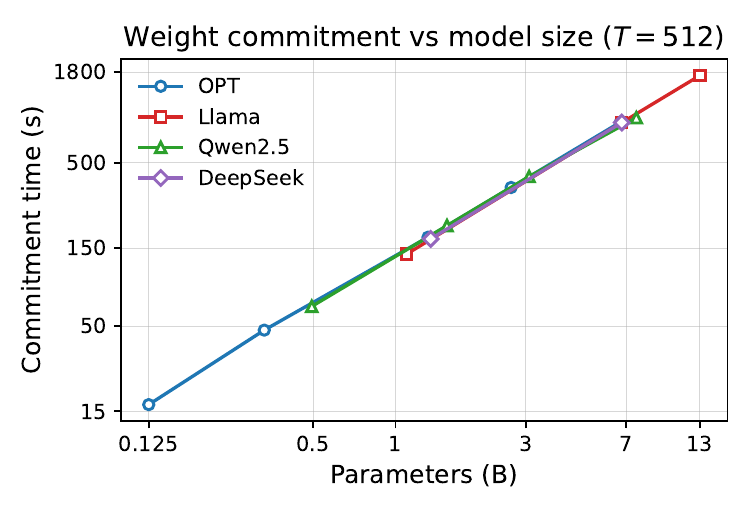}
\hfill
\includegraphics[width=0.49\linewidth]{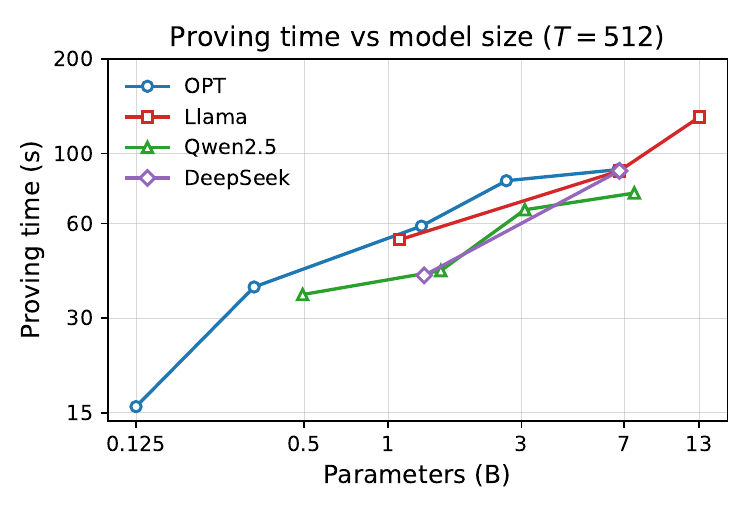}
\caption{Dense-weight commitment and covered proving cost against model size at $T{=}512$. Every linear shape was measured directly on the A100. Nonlinear argument costs are interpolated from measured size curves, and their opening-preparation costs are fitted from measured openings.}
\label{fig:scaling}
\end{figure}

\subsection{Loss Functions and Fine-Tuning}
The certified scalar is the masked next-token negative log-likelihood of~\eqref{eq:llm-nll} over the auditor's sequences, and the audit compares it with a public acceptance band.
Published domain-adaptation results illustrate the available separation: held-out loss falls from $1.32$ to $0.99$ on biomedical text and from $1.63$ to $1.34$ on computer science text, while the same adaptation raises loss on the original distribution from $1.19$ to $1.63$~\citep{gururangan2020dontstop,luo2023forgetting}.
These changes of roughly $0.3$ nats are much larger than the resolution of our fixed-point representation, although choosing an acceptance band for a new task still requires a task-specific calibration set.
A fairness gap, membership check, or safety score changes the public terminal objective and may select different auditor-chosen sequences, while the forward circuit and committed weights remain the same.
The output head accounts for between $0.1$ and $1.5\%$ of the proving time. Public post-processing adds little when it reuses the same forward evaluation, whereas an objective that requires a second evaluation also incurs its additional sequence workload.

\section{Conclusion}

Without a certificate bound to the released checkpoint, a trainer can advertise domain adaptation, instruction tuning, or a safety and fairness property that the model does not attain, or substitute a weaker checkpoint and still pass an evaluation on a known test set. The harm is an accepted false claim about the checkpoint.

We presented zkLLMPoT, a zero-knowledge framework for auditing outcomes attributed to LLM training without proving the full training trajectory. By committing the model before the auditor selects challenge sequences, zkLLMPoT reduces the certified computation to a forward evaluation under an auditor-specified objective, making the proof cost independent of the number of training iterations. Across four LLM families, the proving time is 41–59 seconds at the 1.xB scale and 131 seconds at 13B parameters, with verification below half a second. These results suggest a practical path toward succinct and privacy-preserving auditing of trained LLM checkpoints.

We discuss several future directions: (1) faster operator designs and tighter fusion that extend the same one-shot relation to larger models and longer sequences; (2) more flexible user-defined audits of training outcomes, beyond a fixed catalog of loss, fairness, membership, and safety objectives, while still compiling to the same committed forward circuit; and (3) protocol designs that admit multiple trainers, auditors, or downstream verifiers when several parties jointly hold a checkpoint or jointly set the challenge, without revealing $\theta$ or the private corpus.

\section*{AI use statement}

We used generative AI tools to polish the writing, to help check experimental outputs, and to organize recorded measurements into tables and figures.
We did not use them to generate research ideas, to design the protocol or its gadgets, to formulate or prove claims, or to produce any reported number: every measurement comes from running our own code on hardware we control.
All AI-assisted edits were reviewed by the authors before inclusion, and the use of these tools remained under our direction.
We take responsibility for the final content of this work, including text, claims, and artifacts produced with the aid of generative AI.

\section*{Ethics statement}

This work builds an auditing mechanism, and its purpose is to let a trainer substantiate a claim about a released checkpoint without disclosing the weights or the data the checkpoint was trained on. It does not involve human subjects, and it releases no new dataset. The models and corpora we evaluate on are public releases used under their own licenses. We note that a certificate of the kind zkLLMPoT produces attests to the value of a stated objective on challenge sequences the auditor chooses; it is evidence about a specific, public property, and should not be read as a broader guarantee about a model's safety, provenance, or fitness for a particular use.

\section*{Reproducibility statement}

We will open-source our implementation after the review period, including the protocol code and the scripts needed to reproduce the reported results.

\bibliography{ref}
\bibliographystyle{style2027/iclr2027_conference}

\appendix
\section{Related Work}
\label{sec:related}
\textbf{Zero-knowledge proofs of inference.}
A large body of work certifies that a committed model produced a claimed output, without revealing weights.
The threat these systems target is \emph{inference-time substitution}: an untrusted service may silently serve a cheaper or weaker model than the one it advertised, i.e., ``dumb down'' the deployed model~\citep{chen2024zkml,sun2024zkllm,qu2025zkgpt}.
Early systems target CNNs and generic neural predictors, compiling matrix multiplies and nonlinearities into SNARK-friendly constraints~\citep{liu2021zkcnn,weng2021mystique,hao2024scalable,chen2024zkml}.
More recently, zkLLM and zkGPT scale this idea to Transformer inference, with lookup-style or GKR-style gadgets for softmax, GELU, and attention~\citep{sun2024zkllm,qu2025zkgpt}.
They therefore police serving integrity by checking that a deployed forward pass matches committed architecture and weights, rather than checking how those weights were obtained.

\textbf{Zero-knowledge proofs of training.}
Proofs of training (zkPoT) aim to show that a released checkpoint results from running a declared learning algorithm on committed data.
Garg et al.\ give the first practical zkPoT protocols, but only for logistic regression~\citep{garg2023zkpot}.
Kaizen proves full gradient-descent iterations with backpropagation for deep networks, at a cost of about 15 minutes per step already on VGG-11~\citep{abbaszadeh2024kaizen}.
Confidential-DPproof additionally certifies that training satisfied a claimed DP-SGD guarantee, and reports on the order of 100 hours even on CIFAR-10~\citep{shamsabadi2024confidential}.
All of these systems encode the backward pass and optimizer update into the constraint system.
Under the unified LLM training model of~\secref{sec:prelim-training}, that circuit grows with both model width and the number of steps, which is the reason naive zkPoT is unrealistic for OPT-, Llama-, Qwen2.5-, or DeepSeek-Coder-scale runs.

\textbf{Optimistic and third-party audits.}
A complementary line avoids SNARKs by letting an auditor re-execute training, or by making hardware training deterministic enough to replay~\citep{srivastava2024optimistic,waiwitlikhit2024zkaudit}.
ZKAudit shows that an external party can check training properties (including membership and fairness-style queries) if it may see or re-run the learning process~\citep{waiwitlikhit2024zkaudit}.
These protocols achieve strong semantic coverage, but they either expose proprietary weights and data or require a replica of the original training compute.
They therefore make a different choice in the trade-off between efficiency and privacy that zk-SNARKs are designed to address~(\secref{sec:prelim-zksnark}).

zkLLMPoT instead certifies an auditor-chosen evaluation of committed weights from the decoder-only family $\mathcal{A}$.
The relation excludes the backward pass and $\mathsf{Opt}$, so its cost does not grow with the number of training steps, and the proof reveals neither $\theta$ nor the private corpus.
This lets an auditor check properties of the resulting checkpoint at Transformer scale without replaying training.

\section{Details of the Zero-Knowledge Protocol}
\label{app:zk}

The whole construction rests on one reduction. A claim about a tensor with $2^n$ entries is turned into a claim about a single field element, and every operator of the forward pass is arranged so that verifying it becomes an instance of the multilinear sumcheck protocol. This appendix builds that machinery in order: the encoding that makes it possible (\secref{app:zk-mle}), the sumcheck protocol itself (\secref{app:zk-sumcheck}), the commitment scheme that settles the claims sumcheck leaves behind (\secref{app:zk-pcs}), the principle by which the forward pass is decomposed into instances of the two (\secref{app:zk-reduce}), and the gadget for each operator (\secref{app:zk-ops}).

\subsection{Multilinear extensions}
\label{app:zk-mle}

\textbf{Tensors as tables.}
Every tensor in~\eqref{eq:llm-forward}--\eqref{eq:llm-nll} is a real array, while the argument system works over $\mathbb{F}_p$ with $p$ the scalar field order of BLS12-381.
We fix a public scale $s$ and encode by signed fixed point,
\begin{equation}
\label{eq:app-quant}
  \mathrm{enc}_s(x) \;=\; \bigl\lfloor s\,x \bigr\rceil \bmod p ,
\end{equation}
a negative value being represented by its additive inverse.
A tensor with $N$ entries is flattened, padded to length $2^n$ with $n=\lceil\log_2 N\rceil$, and read as a table
\begin{equation}
\label{eq:app-table}
  T\colon\{0,1\}^{n}\to\mathbb{F}_p ,
  \qquad
  T(b_1,\dots,b_n) \;=\; \mathrm{enc}_s\bigl(\text{entry }\textstyle\sum_k b_k 2^{k-1}\bigr).
\end{equation}

\textbf{The extension.}
A function on the hypercube has a unique extension to $\mathbb{F}_p^n$ of degree at most one in each variable. With
\begin{equation}
\label{eq:app-eq}
  \widetilde{\mathrm{eq}}(b,X) \;=\; \prod_{k=1}^{n}\bigl(b_kX_k+(1-b_k)(1-X_k)\bigr),
\end{equation}
which is $1$ at $X=b$ and $0$ elsewhere on $\{0,1\}^n$, that extension is
\begin{equation}
\label{eq:app-mle}
  \widetilde{T}(X_1,\dots,X_n) \;=\; \sum_{b\in\{0,1\}^{n}} T(b)\,\widetilde{\mathrm{eq}}(b,X).
\end{equation}
We write $T(u)$ for $\widetilde{T}(u)$ throughout.

\textbf{Claims as evaluations.}
Two tables are equal iff their extensions agree at a uniformly random $u\in\mathbb{F}_p^{n}$, up to probability $n/|\mathbb{F}_p|$ by the Schwartz--Zippel lemma.
A statement about $2^n$ entries therefore becomes a statement about one field element, and the entries themselves never have to be sent.
This is the property the rest of the appendix uses: every claim the verifier holds has the form $T(u)=y$ for a committed table $T$, and every gadget's job is to trade such a claim for claims of the same form about the tables one step earlier.

\subsection{The sumcheck protocol}
\label{app:zk-sumcheck}

\textbf{Statement.}
Let $g\colon\mathbb{F}_p^{n}\to\mathbb{F}_p$ be a polynomial that the verifier can evaluate at any single point, and let the prover claim
\begin{equation}
\label{eq:app-sumcheck}
  \sum_{b\in\{0,1\}^{n}} g(b) \;=\; c .
\end{equation}
Computing the left side directly costs $2^n$ evaluations. Sumcheck lets the verifier accept or reject it in $O(n)$ work plus one evaluation of $g$.

\textbf{Protocol.}
\Algref{alg:sumcheck} states the interaction. Each round strips one variable: the prover sends the univariate obtained by summing $g$ over the variables it has not yet fixed, and the verifier replaces its claim by an evaluation of that univariate at a fresh random point.

\begin{algorithm}[t]
\caption{Sumcheck for $\sum_{b\in\{0,1\}^{n}}g(b)=c$}
\label{alg:sumcheck}
\small
\begin{algorithmic}[1]
\Require A polynomial $g\colon\mathbb{F}_p^{n}\to\mathbb{F}_p$ that the verifier can evaluate at a single point; a claimed sum $c$.
\Ensure The verifier accepts or rejects.
\State $c_0\gets c$
\For{$k=1,\dots,n$}
    \State the prover sends $g_k(X)=\sum_{b\in\{0,1\}^{n-k}}g(v_1,\dots,v_{k-1},X,b)$, as its evaluations at $0,1,\dots,\deg g_k$
    \State the verifier rejects unless $g_k(0)+g_k(1)=c_{k-1}$
    \State the verifier samples $v_k\leftarrow\mathbb{F}_p$ and sends it; both parties set $c_k\gets g_k(v_k)$
\EndFor
\State the verifier accepts iff $g(v_1,\dots,v_n)=c_n$
\end{algorithmic}
\end{algorithm}

\textbf{Completeness and soundness.}
Completeness is immediate: an honest prover's $g_k$ satisfies $g_k(0)+g_k(1)=c_{k-1}$ by construction, since splitting the $k$-th variable partitions the remaining sum.
For soundness, if the claim is false then $g_1$ must differ from the honest round polynomial, and two distinct univariates of degree at most $\mu$ agree at no more than $\mu$ points, so $c_1$ is again a false claim except with probability $\mu/|\mathbb{F}_p|$; the argument repeats down the recursion. A union bound over $n$ rounds gives soundness error at most $\mu n/|\mathbb{F}_p|$.

\textbf{Degree and cost.}
In every gadget below, $g$ is a product of $\mu$ multilinear factors, possibly times $\widetilde{\mathrm{eq}}$, so $\deg g_k\le\mu$ and each round polynomial is $\mu+1$ field elements.
The prover folds its tables once per round, halving them, so its total work is $O(\mu\,2^{n})$ field operations; the transcript is $n(\mu+1)$ field elements and the verifier does $O(\mu n)$ work.
The largest $\mu$ anywhere in zkLLMPoT is the softmax digit product of~\secref{app:zk-ops}.

\textbf{Leftover claims.}
The final check needs $g(v_1,\dots,v_n)$, and $g$ is built from the witness tables. The verifier cannot evaluate it unaided, so the protocol ends by handing back claims of the form $T(v)=y$. Those are settled by the commitment scheme of~\secref{app:zk-pcs}. Sumcheck reduces the size of a claim; it never discharges one.

\subsection{The polynomial commitment scheme}
\label{app:zk-pcs}

\begin{definition}[PCS]
\label{def:pcs}
A polynomial commitment scheme for multilinear polynomials over $\mathbb{F}_p^{n}$ is a tuple $(\mathsf{Setup},\mathsf{Commit},\mathsf{Open},\mathsf{Check})$: $\mathsf{Setup}(1^\lambda,n)$ outputs public parameters, $\mathsf{Commit}(\mathsf{pp},T)\to[[T]]$, and $\mathsf{Open}(\mathsf{pp},T,u)\to(y,\pi)$ produces an evaluation with a proof accepted by $\mathsf{Check}(\mathsf{pp},[[T]],u,y,\pi)$.
It is \emph{binding} if no efficient adversary produces $[[T]]$ with accepting proofs for $y\neq y'$ at one $u$, and \emph{hiding} if the commitment and openings reveal nothing beyond the asserted evaluations.
\end{definition}

\textbf{Construction.}
zkLLMPoT instantiates Definition~\ref{def:pcs} with Pedersen commitments \citep{pedersen1991noninteractive} in the matrix layout of Hyrax \citep{wahby2018doubly}.
Split $n=n_1+n_2$ and read the table as a $2^{n_1}\times2^{n_2}$ matrix $M$ with $M[i][j]=T(i\Vert j)$.
$\mathsf{Setup}$ samples generators $g_1,\dots,g_{2^{n_2}}\in\mathbb{G}$ with unknown relative discrete logarithms, and
\begin{equation}
\label{eq:app-commit}
  \mathsf{Commit}(T) \;=\; \bigl(C_1,\dots,C_{2^{n_1}}\bigr),
  \qquad
  C_i \;=\; \sum_{j=1}^{2^{n_2}} M[i][j]\;g_j ,
\end{equation}
one group element per row. Binding follows from discrete log: a second opening of some $C_i$ yields a nontrivial relation among the $g_j$.

\textbf{Opening.}
Because~\eqref{eq:app-mle} factors across the variable split, writing $u=(u^{(1)},u^{(2)})$,
\begin{equation}
\label{eq:app-open}
\begin{aligned}
  T(u)
  &= \sum_{i}\widetilde{\mathrm{eq}}(i,u^{(1)})
     \sum_{j}M[i][j]\,\widetilde{\mathrm{eq}}(j,u^{(2)}) \\
  &= \bigl\langle \mathbf{m},\,\widetilde{\mathrm{eq}}(\cdot,u^{(2)})\bigr\rangle , \\
  \mathbf{m}
  &= \sum_i \widetilde{\mathrm{eq}}(i,u^{(1)})\,M[i][\cdot] .
\end{aligned}
\end{equation}
The prover sends the folded row $\mathbf{m}$; the verifier applies the same fold to the commitments, $\sum_i\widetilde{\mathrm{eq}}(i,u^{(1)})C_i$, using the homomorphism of~\eqref{eq:app-commit}, checks that $\mathbf{m}$ commits to it, and then evaluates the inner product itself.
\textbf{Adding hiding.}
The scheme is binding but not hiding: $C_i$ is a deterministic function of row $i$, so anyone who guesses a row can confirm the guess, and the folded $\mathbf{m}$ is $2^{n_2}$ field elements of the witness sent in the clear.
Both are repaired by randomness. $\mathsf{Setup}$ samples one further generator $h$, again with unknown discrete logarithm relative to the $g_j$, and each row is blinded,
\begin{equation}
\label{eq:app-commit-hiding}
  C_i \;=\; \sum_{j=1}^{2^{n_2}} M[i][j]\,g_j \;+\; \rho_i\,h ,
  \qquad
  \rho_i \leftarrow \mathbb{F}_p .
\end{equation}
Since $\rho_i$ is uniform and $h$ generates the group, $\rho_i h$ is uniform in $\mathbb{G}$ and independent of the row, so $C_i$ is uniformly distributed whatever $M[i][\cdot]$ contains: the commitment carries no information about the table at all, and binding is unaffected because a second opening still yields a nontrivial relation among $g_1,\dots,g_{2^{n_2}},h$.
The homomorphism survives the blinding --- folding by $\widetilde{\mathrm{eq}}(\cdot,u^{(1)})$ carries the blinders to $\rho=\sum_i\widetilde{\mathrm{eq}}(i,u^{(1)})\rho_i$ --- so the verifier still forms the folded commitment of~\eqref{eq:app-open}, now a commitment to $\mathbf{m}$ under blinder $\rho$.
The prover then proves $\langle\mathbf{m},\widetilde{\mathrm{eq}}(\cdot,u^{(2)})\rangle=y$ against that commitment with an inner-product argument rather than sending $\mathbf{m}$, which also reduces the opening from $2^{n_2}$ field elements to $O(n_2)$ group elements.

\textbf{Costs.}
Committing is one multi-scalar multiplication over $2^{n}$ scalars, and is the only part of the protocol whose cost is set by the model rather than by the challenge.
Taking $n_1\approx n_2$ balances the opening: the inner-product proof carries $O(n_2)$ group elements, while the verifier performs $O(2^{n_1}+2^{n_2})$ group operations to fold the row commitments and verify the opening.
Padding is free here in a stronger sense than in~\secref{app:zk-mle}: the commitment is a homomorphic sum, so a padded zero contributes the identity and its term can be skipped outright.

\subsection{Reducing the forward pass to sumchecks and openings}
\label{app:zk-reduce}

The circuit $C_{\mathcal{A},\mathcal{L}}$ is not one large constraint system.
It is the operator graph of~\eqref{eq:llm-forward}--\eqref{eq:llm-block}, one node per operator, and each node carries a gadget that writes its input--output relation in the form of~\eqref{eq:app-sumcheck}.

Proving traverses that graph backwards, one operator at a time.
Suppose the auditor holds a claim $h^{(\ell)}(u_\ell)=y_\ell$ about the output of layer $\ell$.
The gadget for that layer answers it with a sumcheck, whose leftover evaluation splits in two: a claim about the layer's input, and a claim about its weights,
\begin{equation}
\label{eq:app-chain}
  \bigl(u_\ell,\,y_\ell\bigr)
  \;\xrightarrow{\ \text{sumcheck}\ }\;
  \bigl(u_{\ell-1},\,y_{\ell-1}\bigr)
  \;\text{ and }\;
  \theta_\ell(w_\ell)=c_\ell .
\end{equation}
The weight claim is set aside for an opening.
The input claim is an output claim for the layer below, so the same step applies again.
Consecutive gadgets are joined by an equality test: the auditor checks that the pair $(u_{\ell-1},y_{\ell-1})$ emitted at layer $\ell$ is exactly the pair it verifies at layer $\ell-1$ (Step~6 of~\Algref{alg:main}).
These tests are what make the gadgets one proof of the forward pass rather than $L$ unrelated proofs.
After $L$ steps no operator is left, and every claim still open is an evaluation of a table: claims about $S$ are checked on the public input, and claims about committed tables are discharged by~\secref{app:zk-pcs}. Sumcheck shrinks a claim about a whole tensor to a claim about one point, and only the PCS can settle a committed point.

\subsection{Operator gadgets}
\label{app:zk-ops}

Each gadget below states its relation and the sumcheck instance it becomes.
$B$ is the number of rows (batch times sequence), and $u_B,u_I,u_O$ are challenge vectors of the matching lengths.

\textbf{Linear layers and attention products.}
For $Y=XW$ with $X\in\mathbb{F}_p^{B\times I}$ and $W\in\mathbb{F}_p^{I\times O}$, the extensions satisfy
\begin{equation}
\label{eq:app-matmul}
  Y(u_B,u_O) \;=\; \sum_{b\in\{0,1\}^{\log I}} X(u_B,b)\;W(b,u_O),
\end{equation}
an instance of~\eqref{eq:app-sumcheck} with $n=\log I$ and $\mu=2$.
The prover first folds $X$ over $u_B$ and $W$ over $u_O$ at cost $O(BI+IO)$, then runs the sumcheck over the contracted dimension alone, which is why widening $O$ costs far less than widening $I$.
The claims left over are $X(u_B,u_I)$ and $W(u_I,u_O)$.
One gadget serves the projections, the FFN matrices, the output head, and both attention products $QK^{\top}$ and $PV$.

\textbf{Residual connections.}
Addition is linear in the extension, $(A+B)(u)=A(u)+B(u)$, so a residual is not a sumcheck instance: the verifier splits the claim and passes both halves back unchanged.

\textbf{Elementwise nonlinearities.}
GELU, SiLU and ReLU are enforced by a lookup against a public two-column table, the quantized input range $T_{\mathrm{in}}$ and its image $T_{\mathrm{out}}$.
Folding the columns with a random $r$ turns the pair condition into single-column membership,
\begin{equation}
\label{eq:app-fold}
  S_{\mathrm{com}} = S_{\mathrm{in}}+r\,S_{\mathrm{out}},
  \qquad
  T_{\mathrm{com}} = T_{\mathrm{in}}+r\,T_{\mathrm{out}} ,
\end{equation}
and membership becomes a sum by the logarithmic-derivative identity: for random $\beta$,
\begin{equation}
\label{eq:app-logup}
  \sum_{i=1}^{D}\frac{1}{\beta+S_{\mathrm{com}}[i]}
  \;=\;
  \sum_{j=1}^{N}\frac{m[j]}{\beta+T_{\mathrm{com}}[j]} ,
\end{equation}
with $m[j]$ the multiplicity of the $j$-th table entry.
Both sides are instances of~\eqref{eq:app-sumcheck}, run in two phases over $\log(D/N)$ and $\log N$ variables, which is why $D$ must be a multiple of $N$ and why the arguments are batched across the network.
The table is public and the verifier recomputes $T_{\mathrm{com}}$ itself: a prover free to choose the table would be free to choose the activation.

\textbf{Normalization.}
LayerNorm and RMSNorm share one gadget with three parts.
The row statistic is a weighted inner product, an instance of~\eqref{eq:app-sumcheck} with $\mu=3$ against the $\widetilde{\mathrm{eq}}$ weights,
\begin{equation}
\label{eq:app-sumsq}
  \sigma(u_B) \;=\; \sum_{r,i}\widetilde{\mathrm{eq}}(r,u_B)\;\bar{x}[r][i]^{2},
  \qquad
  \bar{x} = x \;\text{ or }\; x-\mathrm{mean}(x),
\end{equation}
where LayerNorm's mean is a linear functional of $x$ and so costs one extra evaluation rather than a sumcheck.
Bringing $\sigma$ into the table's range is a proved shift whose remainder is range-checked by~\eqref{eq:app-logup}, and the normalizer is then a range-mapped lookup of $1/\sqrt{\cdot}$ --- one entry per row, not per element.
The output is the broadcast product $y[r][i]=\bar{x}[r][i]\,\iota[r]$, a second sumcheck of the same shape as~\eqref{eq:app-sumsq}.

\textbf{Softmax.}
Exponentiation is decomposed rather than tabulated.
Each row of the score matrix is shifted so its entries are non-positive, and the shifted value is written in mixed radix, $\hat{x}=\sum_{k=1}^{K}B_k d_k$ with $B_k=\prod_{k'<k}b_{k'}$.
Digit $k$ carries its own public table
\begin{equation}
\label{eq:app-softmax-table}
  T_k[d] \;=\; \Bigl\lfloor \theta_k\exp\!\bigl(-B_k\,d/(s\sqrt{d_h})\bigr) \Bigr\rceil ,
\end{equation}
so that the digit images multiply back to the exponential,
\begin{equation}
\label{eq:app-softmax-prod}
  y \;=\; \prod_{k=L+1}^{K} T_k[d_k] .
\end{equation}
The low digits fall below the output precision and are range-checked only; the top digit saturates.
Each digit is a lookup~\eqref{eq:app-logup}, and~\eqref{eq:app-softmax-prod} is a Hadamard product whose sumcheck has $\mu=K-L+1$, the largest degree in the protocol and the source of the $d=4$ used in~\secref{app:zk-soundness}.
All heads of a layer go into one argument, sharing the digit tables and the per-round cost.

\subsection{Soundness}
\label{app:zk-soundness}

A prover can cheat in three ways: a false sumcheck, a false lookup, or a second opening of one commitment. Three lemmas bound each in turn, and a union bound over the arguments the circuit contains turns them into a statement about the certificate. Throughout, $\varepsilon$ is the probability that the auditor accepts a false claim, over its own coins.

\begin{lemma}[Sumcheck, \citealp{lund1992algebraic}]
\label{lem:sumcheck}
Let $g\colon\mathbb{F}_p^{n}\to\mathbb{F}_p$ have degree at most $\mu$ in every variable, and suppose $\sum_{b\in\{0,1\}^{n}}g(b)\neq c$. Then for any prover strategy, the verifier of \Algref{alg:sumcheck} accepts with probability at most $\mu n/|\mathbb{F}_p|$.
\end{lemma}

\begin{lemma}[Commitment binding, \citealp{pedersen1991noninteractive}]
\label{lem:binding}
An adversary that outputs a commitment $(C_1,\dots,C_{2^{n_1}})$, a point $u$, and two accepting openings with $y\neq y'$ in time $t$ yields a nontrivial relation $\sum_j\gamma_j g_j+\gamma_h h=0$ in time $t+O(2^{n})$. The binding advantage of~\eqref{eq:app-commit-hiding} is therefore at most $\mathrm{Adv}^{\mathrm{DL}}_{\mathbb{G}}(t)+\varepsilon_{\mathrm{ipa}}$, with $\varepsilon_{\mathrm{ipa}}$ the soundness error of the inner-product argument.
\end{lemma}

\begin{lemma}[Lookup argument]
\label{lem:lookup}
Let $S_{\mathrm{in}},S_{\mathrm{out}}\in\mathbb{F}_p^{D}$ and $T_{\mathrm{in}},T_{\mathrm{out}}\in\mathbb{F}_p^{N}$, and suppose the witness rows $\{(S_{\mathrm{in}}[i],S_{\mathrm{out}}[i])\}_{i\le D}$ are \emph{not} all table rows. Then the argument of~\eqref{eq:app-fold}--\eqref{eq:app-logup} accepts with probability at most
\begin{equation}
\label{eq:app-lookup-eps}
  \varepsilon_{\mathrm{lk}}
  \;\le\;
  \frac{DN}{|\mathbb{F}_p|}
  \;+\;
  \frac{D+N}{|\mathbb{F}_p|}
  \;+\;
  \frac{4\log D+4\log N}{|\mathbb{F}_p|} ,
\end{equation}
over the folding challenge $r$, the shift $\beta$, and the sumcheck challenges.
\end{lemma}

Lemma~\ref{lem:sumcheck} is notable for what it omits: only the round degree and the round count enter, never the size of the table, so security is nearly independent of model size and the degree of the softmax gadget is the one structural choice worth watching. Lemma~\ref{lem:binding} is Pedersen binding; what the matrix layout of~\secref{app:zk-pcs} adds is that the fold $C(u^{(1)})=\sum_i\widetilde{\mathrm{eq}}(i,u^{(1)})C_i$ is public, so two openings at one $u$ land on the same group element and subtract to a relation among the generators. It turns ``the prover asserted $T(u)=y$'' into ``$T(u)=y$ for the table fixed at commitment time'', and it is the only one of the three bounds that is computational rather than statistical. Lemma~\ref{lem:lookup} has one term per challenge: $DN$ for a folded collision, $D+N$ for the logarithmic-derivative identity \citep{habock2022logup}, and the rounds of the two auxiliary sumchecks. The $DN$ term dominates and is the price of folding with one global $r$ rather than one challenge per row --- below $2^{-220}$ at our sizes, so the cheap version is the right one.

\begin{theorem}[Soundness of \Algref{alg:main}]
\label{thm:soundness}
Let $\theta$ be the checkpoint committed in Step~2 and $S$ the sequences the auditor draws afterwards. A prover that makes the auditor accept a value $v\neq\mathcal{L}(f_{\mathcal{A},\theta}(S))$ succeeds with probability at most
\begin{equation}
\label{eq:app-total-eps}
  \varepsilon \;\le\; \sum_{a\in\mathcal{C}_{\mathcal{A},\mathcal{L}}}\varepsilon_a
  \;+\; \varepsilon_{\mathrm{ipa}}^{\mathrm{tot}}
  \;+\; \mathrm{Adv}^{\mathrm{DL}}_{\mathbb{G}}(t) ,
\end{equation}
where $\varepsilon_a$ is the sumcheck error of a linear gadget or the lookup error of a nonlinear gadget, and $\varepsilon_{\mathrm{ipa}}^{\mathrm{tot}}$ is the aggregate soundness error of all inner-product opening arguments.
\end{theorem}

\begin{proof}
Acceptance is the conjunction of three things: every operator argument accepts, every chaining equality of~\eqref{eq:app-chain} holds, leftover claims on the public $S$ match the auditor's table, and every leftover claim on a committed table opens against a commitment published in Step~2. Order the operators in reverse topological order and walk the graph from the output. The auditor's top claim is false by assumption. At each operator, either its argument accepted a false claim --- an event bounded by $\varepsilon_a$ --- or it emitted a false claim on one of its inputs or on its weights. Because Step~6 checks that the emitted pair $(u_{\ell-1},y_{\ell-1})$ is \emph{syntactically} the pair the predecessor is asked to justify, the falsity is handed to one determined predecessor rather than dispersed. The recursion therefore terminates after at most $L$ steps at a claim about a committed table or about the public $S$; the former requires the event of Lemma~\ref{lem:binding}, and the latter is checked directly. Summing the per-argument errors along the path, adding the aggregate IPA opening error, and accounting for the discrete-log binding advantage gives~\eqref{eq:app-total-eps}.
\end{proof}

\textbf{Numerical soundness error.}
Over BLS12-381, $|\mathbb{F}_p|\approx2^{255}$. The largest round degree anywhere in zkLLMPoT is the softmax digit product of~\eqref{eq:app-softmax-prod}, $\mu=4$, and the longest argument is that same product over all heads of a layer, $2^{23}$ entries and so $n=23$ rounds; by Lemma~\ref{lem:sumcheck} no single sumcheck exceeds $92/|\mathbb{F}_p|<2^{-248}$. The lookups are looser but not by much: the widest is the activation table over $D=2^{23}$ elements against $N=2^{12}$ entries, giving $DN/|\mathbb{F}_p|<2^{-220}$ in~\eqref{eq:app-lookup-eps}. Certifying OPT-1.3B at $T{=}512$ runs $1926$ arguments, Llama-1.1B $1786$, DeepSeek-Coder-1.3B $1180$ and Qwen2.5-1.5B $1166$, so the statistical sum $\sum_a\varepsilon_a$ in~\eqref{eq:app-total-eps} stays below $2^{-209}$ for every configuration.

The statistical error is therefore negligible relative to the two computational terms in~\eqref{eq:app-total-eps}. If the aggregate IPA error is at most $2^{-128}$ and the discrete-log advantage on $\mathbb{G}_1$ is about $2^{-128}$, the additive bound is below $2^{-127}+2^{-209}$, corresponding to approximately $127$ bits of security under this conservative union bound.

\subsection{Zero knowledge}
\label{app:zk-hiding}

The auditor is entitled to $v$ and to nothing else: not $\theta$, not any intermediate activation, and not the data the checkpoint was trained on, which enters the transcript only through $\theta$ and so is hidden whenever $\theta$ is.
Zero knowledge in zkLLMPoT reduces to the hiding property of the commitment scheme.
Sumcheck never transmits witness values: a round polynomial is a claim about a sum, the verifier's reply is a random point, and the exchange ends by handing back an evaluation claim $T(u)=y$ (\secref{app:zk-sumcheck}). Each round polynomial is in addition masked to be uniform among everything consistent with the check the verifier is about to perform, at the cost of one extra committed polynomial per argument \citep{chiesa2017zksumcheck,xie2019libra}.
What travels outside the commitments is public and independent of $\theta$: the architecture $\mathcal{A}$, the objective $\mathcal{L}$, the sequences $S$, the quantization parameters, and every lookup table, all of which the auditor rebuilds itself. Committed are the weights, every activation, and every data-dependent constant, including the per-row softmax shift, which is proved in range rather than revealed.
Every table the auditor learns anything about, it therefore learns about through $\mathsf{Commit}$ and $\mathsf{Open}$, and it remains to discuss how these two hide.

\textbf{Blinding the commitment.}
Committing exposes the whole weight tensor to the auditor, so it is where the randomness goes first. A row of the matrix layout is committed with a blinding generator, $C_i=\sum_j M[i][j]g_j+\rho_i h$ as in~\eqref{eq:app-commit-hiding}, with $\rho_i$ drawn fresh and uniform per row. Because $\rho_i$ is uniform and $h$ generates the group, $\rho_i h$ is uniform in $\mathbb{G}$ and independent of the row it blinds, so $C_i$ is uniformly distributed whatever the row contains. This matters concretely: a quantized weight row has a small effective range, so without the blinder an auditor who suspects a particular row could recompute its commitment and confirm the guess.

\textbf{Blinding the opening.}
An opening has to convince the auditor of $T(u)=y$ without saying more than $y$, and two things make that possible. The blinding is homomorphic, so the verifier's public fold carries the row blinders to a single $\rho=\sum_i\widetilde{\mathrm{eq}}(i,u^{(1)})\rho_i$ and the folded object is again a hiding commitment, this time to the folded row $\mathbf{m}$. And $\mathbf{m}$ is never sent: the prover proves $\langle\mathbf{m},\widetilde{\mathrm{eq}}(\cdot,u^{(2)})\rangle=y$ against that commitment with an inner-product argument. Hiding then follows because $\rho$ is a uniform unknown entering exactly one of the verifier's equations, so every value the public claim could take admits a consistent $\rho$: the auditor's view can be reconstructed from the public $u$ and $y$ alone, and therefore says nothing about the table. Since the blinders are uniform field elements, this hiding is perfect rather than computational, and binding by Lemma~\ref{lem:binding} is the computational half of the pair.

\end{document}